\documentclass[twoside,11pt]{article}

\usepackage{jmlr2e}
\usepackage{lastpage}

\usepackage{amsfonts,mathtools}
\usepackage{booktabs,array}
\usepackage{enumitem}
\usepackage{xcolor}
\usepackage{tikz}
\usetikzlibrary{arrows.meta,positioning,fit,shapes.geometric,calc}
\hypersetup{hidelinks}
\newcolumntype{L}[1]{>{\raggedright\arraybackslash}p{#1}}

\newcommand{\R}{\mathbb{R}}

\newcommand{\E}{\mathbb{E}}
\newcommand{\Prob}{\mathbb{P}}
\newcommand{\calX}{\mathcal{X}}

\newcommand{\calP}{\mathcal{P}}
\newcommand{\calQ}{\mathcal{Q}}
\newcommand{\weakto}{\rightharpoonup}

\newcommand{\calF}{\mathcal{F}}

\newcommand{\calH}{\mathcal{H}}
\newcommand{\calS}{\mathcal{S}}
\newcommand{\calC}{\mathcal{C}}

\newcommand{\calL}{\mathcal{L}}
\newcommand{\calR}{\mathcal{R}}

\newcommand{\diam}{\mathrm{diam}}
\newcommand{\dist}{\mathrm{dist}}

\newcommand{\esssup}{\operatorname*{ess\,sup}}
\newcommand{\essinf}{\operatorname*{ess\,inf}}

\newcommand{\D}{\mathsf{D}}
\newcommand{\W}{\mathsf{W}}

\newcommand{\CROTS}{\mathsf{CROTS}}
\newcommand{\eps}{\varepsilon}
\newcommand{\dd}{\,\mathrm{d}}
\newcommand{\soft}{\mathrm{soft}}
\newcommand{\hard}{\mathrm{hard}}

\newtheorem{assumption}[theorem]{Assumption}

\jmlrheading{1}{2026}{1-\pageref{LastPage}}{May 2026}{--}{26-0000}{Lei Luo and Jian Yang}
\ShortHeadings{Conditional Random Ordered Transport Spaces}{Luo and Yang}
\firstpageno{1}

\begin{document}

\title{Conditional Random Ordered Transport Spaces}

\author{\name Lei Luo \email cslluo@njust.edu.cn \\
\addr PCA Lab, Key Lab of Intelligent Perception and Systems for High-Dimensional Information of Ministry of Education\\
School of Computer Science and Engineering\\
Nanjing University of Science and Technology\\
Nanjing, China
\AND
\name Jian Yang \email csjyang@njust.edu.cn \\
\addr PCA Lab, Key Lab of Intelligent Perception and Systems for High-Dimensional Information of Ministry of Education\\
School of Computer Science and Engineering\\
Nanjing University of Science and Technology\\
Nanjing, China}

\editor{To be assigned}

\maketitle
\begin{abstract}%
A small Wasserstein distance does not certify that a transformation is admissible. In evidence-constrained, semantic, causal, physical, monotone, or risk-sensitive learning, one must ask not only how far two probability laws are, but whether mass has moved in a direction allowed by available information. We introduce conditional random ordered transport spaces (CROTS), a class of \(L^0\)-valued spaces of random probability measures equipped with a Wasserstein ambient metric, a closed stochastic order, hard and soft ordered transport discrepancies, and a conditional risk functional for evaluating order violation under an evidence sigma-field. The central object is an order-admissible transport geometry for random measure-valued dynamics, distinct from cone-valued metrics, ordered Kantorovich constructions, random Wasserstein spaces alone, and model-specific residuals for generative paths. We develop the foundations of CROTS as a space theory for reliable distributional learning. The results include well-posedness and duality for hard and soft ordered transport, soft-to-hard variational convergence, measurability and completeness of the random lifted space, reductions to classical Wasserstein and ordered geometries, ordered geodesics, constrained barycenters and projections, conditional risk-transport duality, and separation of order-violating distributions. The main stability theorem shows that random learning dynamics may converge in the ambient Wasserstein metric while its local admissibility leakage follows a separate conditional order-risk recursion. The resulting asymptotic order-risk floor provides a mathematical language for evidence overreach, ordered distribution shift, robustness failure, and admissible distributional dynamics.
\end{abstract}

\begin{keywords}
Optimal transport, Wasserstein geometry, stochastic order, random metric spaces, conditional risk measures, fixed point theory, robust learning, evidence constraints.
\end{keywords}

\section{Introduction}
\label{sec:intro}

Machine learning often advances by changing the ambient mathematical space.  Linear prediction is organized by Hilbert and Banach spaces.  Kernel methods use reproducing kernel Hilbert spaces and their Banach or native-space variants, and distributional learning became geometric through Wasserstein spaces \citep{zhang2009rkbs,emery2025native}.  Recent JMLR work has continued this space-driven line through Wasserstein projections, conditional Wasserstein distances, functional optimal transport, and sliced-Wasserstein flows on non-Euclidean spaces \citep{gunsilius2024tangential,chemseddine2025conditional,zhu2024functional,bonet2025sliced}.  We argue that reliable distributional learning requires another change of language: a space in which probability laws are random, transport has admissible directions, and operator stability is evaluated conditionally on evidence or environment.

Classical Wasserstein geometry asks how far two probability measures are.  It gives metrics, geodesics, barycenters, displacement convexity, and gradient flows \citep{villani2003topics,villani2009optimal,ambrosio2008gradient,santambrogio2015optimal}.  But distance alone does not answer a different and often more important question: \emph{is the transport admissible?}  A small displacement may still cross an evidence boundary, violate a semantic monotonicity constraint, break a causal order, or move mass in a physically impossible direction.  Conversely, a large displacement may be safe if it follows the permitted order.  This phenomenon is invisible to ordinary scalar transport cost.

The missing question appears across reliable learning.  In medical decision systems, an output should not move outside patient-evidence and disease-progression constraints.  In robust image regression and face recognition, low-rank matrix regression and matrix-variate noise models have already shown that occlusion, illumination, and dependent noise are better handled as structured perturbations than as independent scalar errors \citep{yang2017nuclear,luo2017robust}.  In representation learning, an update should not destroy task-relevant order.  In robust learning, contamination may be harmless when it lies in a task-admissible direction but harmful when it crosses an order boundary.  In multimodal reasoning, an output should not exceed the evidence cone generated by observations and side information.  These phenomena are not adequately described by the statement $W_p(\mu,\nu)$ is small.  They require a geometry of \emph{ordered admissible transport}.

The central thesis of the paper is the following.
\begin{quote}
\emph{Reliable distributional learning is not merely minimization of transport distance.  It is stability of random measure-valued operators under conditionally admissible ordered transport.}
\end{quote}
This thesis leads to a mathematical object that we call a \emph{conditional random ordered transport space}.  The word ``conditional'' is essential: admissibility is usually evaluated relative to evidence, environment, task context, sensor state, or patient information.  The word ``random'' is essential: data laws, transformations, and environments are stochastic.  The word ``ordered'' is essential: not every low-cost movement is legitimate.  The word ``transport'' is essential: the objects are probability measures and the geometry is induced by couplings.

\paragraph{CROTS at a glance.}
The formal definition appears in Section~\ref{sec:CROTS}.  Informally, a CROTS is the structured space
\[
\boxed{
\CROTS_{p,q,\lambda,\rho}
=\big(\calS_{p,q},\mathbf W_{p,q},\preceq_{\mathrm{a.s.}},
\D^{0,\soft}_{p,\lambda,\preceq},\calR^{\preceq}_{\lambda,\rho}\big),
}
\]
where \(\calS_{p,q}\) is a space of random probability measures, \(\mathbf W_{p,q}\) is the complete ambient random Wasserstein metric, \(\preceq_{\mathrm{a.s.}}\) is an almost-sure stochastic order, \(\D^{0,\soft}_{p,\lambda,\preceq}\) is an \(L^0\)-valued ordered transport discrepancy, and \(\calR^{\preceq}_{\lambda,\rho}\) is the conditional risk used to evaluate order violation.  This early display is intended to make the new ambient space visible before the individual components are developed.

\paragraph{Guiding example.}
Let $\calX=\R^d$ and let $K\subseteq \R^d$ be a closed convex cone.  Define
\[
        x\preceq_K y \quad \Longleftrightarrow \quad y-x\in K .
\]
The cone may represent a semantic direction, a disease-progression direction, a risk-increasing direction, or an evidence-admissible direction.  Given distributions $\mu,\nu\in\calP_p(\R^d)$, ordinary Wasserstein transport searches over all couplings $\pi\in\Pi(\mu,\nu)$.  Ordered transport searches only over couplings satisfying $x\preceq_K y$ almost surely.  The hard ordered cost is therefore directed and may be infinite in one direction.  Exact admissibility is often too rigid for learning, so we also introduce a soft penalty that measures distance to the order graph.  A central theorem below shows that this soft problem converges to hard ordered transport as the penalty diverges.  Thus the soft penalty is a variational relaxation of exact admissibility, not an ad hoc add-on.

\paragraph{How to read the paper.}
Readers who want the conceptual picture should begin with Figure~\ref{fig:framework} and Table~\ref{tab:theory-map}.  The figure shows how four existing languages--Wasserstein geometry, stochastic order, random metric theory, and conditional risk--enter CROTS.  The table gives a map of the theoretical blocks, what each block proves, and what each result means.  The main text gives the definitions, statements, and explanations; the appendices contain proof details.  This layout follows the space-driven style common in JMLR papers on RKBS, native spaces, Wasserstein projections, and conditional Wasserstein geometry \citep{zhang2009rkbs,emery2025native,gunsilius2024tangential,chemseddine2025conditional}.

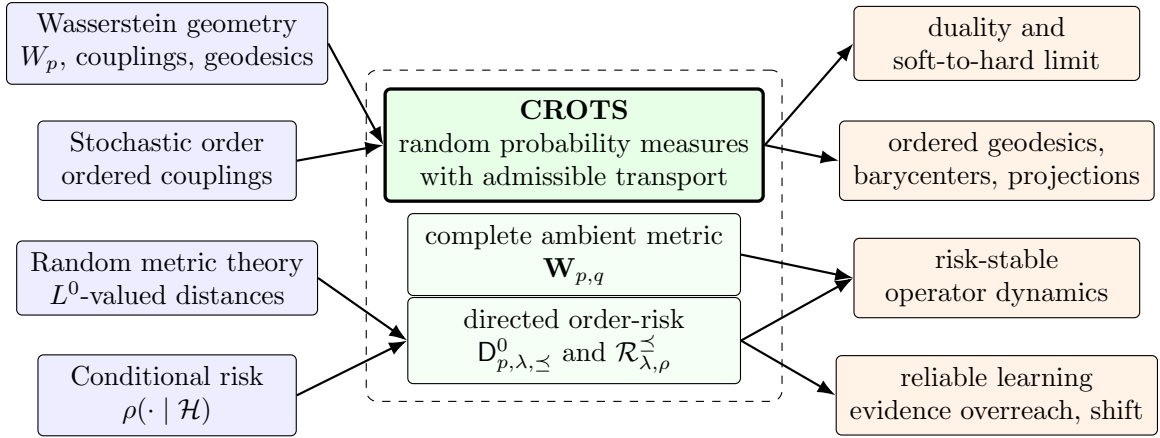
\begin{figure}[t]
\centering
\resizebox{\linewidth}{!}{%
\begin{tikzpicture}[
  font=\small,
  box/.style={draw, rounded corners=2pt, align=center, minimum height=0.72cm, inner xsep=5pt, inner ysep=4pt},
  source/.style={box, fill=blue!7, minimum width=3.3cm},
  core/.style={box, fill=green!10, minimum width=4.25cm, minimum height=0.9cm, very thick},
  mid/.style={box, fill=green!4, minimum width=4.25cm},
  output/.style={box, fill=orange!10, minimum width=3.65cm},
  arr/.style={-Latex, thick}
]

\node[source] (w) at (0,2.25) {Wasserstein geometry\\$W_p$, couplings, geodesics};
\node[source] (o) at (0,0.75) {Stochastic order\\ordered couplings};
\node[source] (r) at (0,-0.75) {Random metric theory\\$L^0$-valued distances};
\node[source] (c) at (0,-2.25) {Conditional risk\\$\rho(\cdot\mid\calH)$};

\node[core] (crots) at (5.2,0.95) {\textbf{CROTS}\\random probability measures\\with admissible transport};
\node[mid] (metric) at (5.2,-0.45) {complete ambient metric\\$\mathbf W_{p,q}$};
\node[mid] (risk) at (5.2,-1.55) {directed order-risk\\$\D^0_{p,\lambda,\preceq}$ and $\calR^{\preceq}_{\lambda,\rho}$};

\node[output] (dual) at (10.6,2.25) {duality and\\soft-to-hard limit};
\node[output] (geo) at (10.6,0.75) {ordered geodesics,\\barycenters, projections};
\node[output] (fp) at (10.6,-0.75) {risk-stable\\operator dynamics};
\node[output] (learn) at (10.6,-2.25) {reliable learning\\evidence overreach, shift};

\draw[arr] (w.east) -- (crots.west);
\draw[arr] (o.east) -- (crots.west);
\draw[arr] (r.east) -- (risk.west);
\draw[arr] (c.east) -- (risk.west);
\draw[arr] (crots.east) -- (dual.west);
\draw[arr] (crots.east) -- (geo.west);
\draw[arr] (risk.east) -- (fp.west);
\draw[arr] (risk.east) -- (learn.west);
\draw[arr] (metric.east) -- (fp.west);

\node[draw, dashed, rounded corners=4pt, fit=(crots)(metric)(risk), inner sep=0.22cm] {};
\end{tikzpicture}%
}
\caption{Conceptual architecture of CROTS.  Existing theories contribute ingredients, but the new space separates the complete ambient metric used for convergence from the directed order-risk functionals used for admissibility.}
\label{fig:framework}
\end{figure}

\begin{table}[t]
\centering
\caption{Theory map of the paper.  The table is intended as a compact guide to the role of each theoretical block.}
\label{tab:theory-map}
\small
\begin{tabular}{L{0.18\linewidth} L{0.24\linewidth} L{0.25\linewidth} L{0.22\linewidth}}
\toprule
Block & Main object & Main result & What it expresses \\
\midrule
Ordered transport & $\Pi_{\preceq}$, $\W^{\hard}_{p,\preceq}$ & Existence, closedness, directed triangle & Reachability is different from distance \\
Soft violation & $\W^{\soft}_{p,\lambda,\preceq}$, $V_{\preceq}$ & Soft-to-hard convergence & Differentiable penalties approximate exact admissibility \\
Non-redundancy & Dirac and small-shift examples & Wasserstein closeness cannot certify admissibility & Ordinary Wasserstein geometry misses order violation \\
Random lift & $\mu:\Omega\to\calP_p(\calX)$ & Measurability and completeness of $\calS_{p,q}$ & Ordered transport distances are retained as random variables \\
CROTS & $(\calS_{p,q},\mathbf W_{p,q},\D^0,\calR)$ & Reduction theorem and risk topology & The new space contains Wasserstein, ordered, random, and Dirac layers \\
Geometry & ordered geodesics, barycenters, projections & Existence under order-compatible assumptions & The space supports geometric objects, not only costs \\
Operators & random measure-valued maps & Fixed points and order-risk leakage bounds & Local admissibility leakage gives a global risk floor \\
Learning consequences & evidence sets and shifts & separation and generalization bounds & Hallucination, robustness, and shift become order-risk phenomena \\
\bottomrule
\end{tabular}
\end{table}

\paragraph{Contribution.}
This paper first constructs the deterministic ordered transport layer.  Starting from a closed ordered Polish space $(\calX,d,\preceq)$, we define ordered couplings, hard directed ordered transport, a soft order-penalized transport problem, and a violation functional.  We prove existence of optimal plans, order-closedness, a directed triangle inequality on comparable triples, Kantorovich duality, and soft-to-hard variational convergence.  These results clarify the basic message that admissibility is not a topological consequence of small Wasserstein distance.

The second contribution is the space itself.  We lift ordered transport to random probability measures and introduce $L^0$-valued ordered transport discrepancies.  These random discrepancies are then evaluated through conditional risk functionals, yielding the conditional random ordered transport space
\[
\CROTS_{p,q,\lambda,\rho}
  = (\calS_{p,q},\mathbf W_{p,q},\preceq_{\mathrm{a.s.}},
     \D^{0,\soft}_{p,\lambda,\preceq},\calR^{\preceq}_{\lambda,\rho}).
\]
The complete metric $\mathbf W_{p,q}$ controls convergence and fixed points, while the directed $L^0$-valued discrepancy and its conditional risk control admissibility.  This separation is central: a directed admissibility risk is not forced to be a metric.

The third contribution is the operator theory and its learning interpretation.  We prove ordered geodesic, barycenter, and projection results, then establish fixed point theorems for order-preserving and approximately order-preserving random measure-valued operators.  The main stability principle says that if local conditional order-risk leakage is persistent, then it creates an explicit asymptotic floor in the order-risk process, even when the ambient Wasserstein error is controlled separately.  This gives a mathematical language for evidence overreach, ordered distribution shift, and robustness failure.

\section{Background and Positioning}
\label{sec:background}

This section places the paper relative to the closest literatures and fixes notation.  The point of the related work discussion is not to claim that order, randomness, risk, or fixed points are new separately.  It is to explain why their combination in a single transport space is not already covered by any one of those literatures.

Optimal transport gives the ambient geometry of probability measures \citep{villani2009optimal,santambrogio2015optimal,peyre2019computational}.  Its variants have become useful in metric learning and representation learning, including smoothed Wasserstein metric learning and data-adaptive spherical sliced-Wasserstein distances \citep{xu2018multi,zhang2025dssw}.  Stochastic order compares probability measures through monotone functions or order-supported couplings \citep{strassen1965existence,shaked2007stochastic,muller2002stochastic}.  The closest ordered-Wasserstein predecessor is \citet{fritz2020stochastic}, who study the ordered Kantorovich monad and the lifting of order from a metric space to probability measures.  Wasserstein projection theory in stochastic order, for example \citet{kim2024backward}, develops deep projection results onto order cones.  Our work uses ordered couplings but studies a different object: random measure-valued operators, soft order violation, conditional risk, and long-run stability.

Random metric theory provides the second ingredient.  In random normed and random locally convex modules, distances and norms are naturally $L^0$-valued quantities, and conditional risk measures are a major application area \citep{guo2010random,guo2011recent}.  We import this lesson into transport geometry: the ordered transport discrepancy between two random probability measures is kept as a random variable before being evaluated by conditional risk.  This differs from random Wasserstein frameworks, such as the recent $L^2$ over Wasserstein viewpoint of \citet{passeggeri2026l2}, which model uncertainty of random probability measures but do not include admissible order, violation penalties, or order-risk operator stability.

The cone example in the introduction might suggest cone metric spaces, but our use of order is different.  Cone metric spaces replace real-valued distances by cone-valued distances \citep{huang2007cone}.  Many fixed point statements in that literature can be scalarized or metrized \citep{asadi2011metrizability}.  CROTS does not use cones as value spaces for distance.  It uses orders or cones to constrain transport directions.  This is why the theory produces directed reachability, violation penalties, ordered geodesics, and order-risk stability rather than another cone-valued fixed point theorem.  Order-preserving Wasserstein distances for sequence matching \citep{su2017order,su2019order} are also related, but they are task-specific algorithmic constructions rather than an $L^0$-valued conditional random transport space.

Finally, this paper is intentionally distinct from two other lines of work.  It is not a theory of self-consistent generative paths, random regularized transport corrections, path residuals, model reductions, or adaptive sampling.  Recent image editing, image-to-image translation, attention-based editing, and rectified-flow works study such model-specific mechanisms and how they shape concrete generative paths \citep{dai2025erddci,zhang2025dcnot,yu2025ssaim,dai2025vrfno}.  Those questions motivate some applications of admissibility, but they are not the object studied here.  The present paper is also not deterministic cone-compatible Monge geometry: it does not characterize Mahalanobis cone acuteness or cone-chain quantile formulas.  Those are deterministic closed-form OT questions.  The present paper is a space-theoretic and operator-theoretic contribution.

\begin{table}[t]
\centering
\caption{Boundary with the closest literature.  The comparison is by primary object, not by a checklist of superiority.}
\label{tab:boundary}
\small
\begin{tabular}{L{0.23\linewidth} L{0.31\linewidth} L{0.36\linewidth}}
\toprule
Closest line & Primary object & Difference of CROTS \\
\midrule
Ordered Kantorovich / stochastic order & lifting an order to probability measures & adds random measure-valued spaces, soft violation, conditional risk, and operator stability \\
Random Wasserstein & random probability measures under Wasserstein uncertainty & adds admissible order, directed discrepancy, and conditional risk \\
Cone metric fixed points & cone-valued distance functions & order constrains couplings and directions, not the value space of a metric \\
Stochastic-order projections & projection onto measure-level order cones & projection is one object; the main theory is a full random ordered transport space \\
Order-preserving sequence OT & temporal matching regularizer & closed orders and conditional random spaces, not sequence-specific matching \\
Path self-consistency methods & corrected generative paths and path residuals & does not use path-space correction operators, path residuals, or model-specific reduction taxonomies \\
Cone-compatible Monge geometry & deterministic cone-cost compatibility and closed forms & no cone-chain quantile formula; the focus is random, conditional, and operator-theoretic \\
\bottomrule
\end{tabular}
\end{table}

Let $(\calX,d)$ be a Polish metric space and $p\in[1,\infty)$.  Denote by $\calP_p(\calX)$ the set of Borel probability measures with finite $p$th moment.  The Wasserstein distance is
\[
 W_p(\mu,\nu)^p=\inf_{\pi\in\Pi(\mu,\nu)}\int d(x,y)^p\dd\pi(x,y),
\]
where $\Pi(\mu,\nu)$ is the set of couplings of $\mu$ and $\nu$.

A partial order $\preceq$ on $\calX$ has graph
\[
        G_{\preceq}=\{(x,y)\in\calX\times\calX:x\preceq y\}.
\]
The order is called closed if $G_{\preceq}$ is closed in $\calX\times\calX$.

\begin{assumption}[Standing ordered transport assumptions]
\label{ass:standing}
Throughout Sections~\ref{sec:deterministic}-\ref{sec:operators}, $(\calX,d)$ is Polish, $p\in[1,\infty)$, and $\preceq$ is a closed partial order on $\calX$.  Products are equipped with the sum metric
\[
  d_2((x,y),(x',y'))=d(x,x')+d(y,y').
\]
The order-violation function is
\[
  \eta_{\preceq}(x,y)=\dist_{d_2}((x,y),G_{\preceq}).
\]
\end{assumption}

Because the order is reflexive, $(x,x)\in G_{\preceq}$.  Therefore
\[
   \eta_{\preceq}(x,y) \le d_2((x,y),(x,x))=d(x,y),
\]
and consequently
\[
   d(x,y)^p \le c_\lambda(x,y) \le (1+\lambda)d(x,y)^p
\]
for every finite $\lambda\ge0$.  This elementary bound is used throughout the paper to keep the soft ordered cost within the usual finite-moment regime.  The closedness assumption is the minimal topological assumption needed for order to survive limits of couplings; it is the reason stochastic dominance remains closed under Wasserstein convergence later in the paper.

\section{Ordered Transport}
\label{sec:deterministic}

This section develops the deterministic fiber of the theory.  It is the ordered counterpart of the usual optimal transport problem on $\calP_p(\calX)$.  The random and conditional layers in the next section are built by applying these deterministic constructions pointwise in $\omega$.

\begin{definition}[Ordered coupling and induced stochastic order]
For $\mu,\nu\in\calP_p(\calX)$ define
\[
     \Pi_{\preceq}(\mu,\nu)=\{\pi\in\Pi(\mu,\nu):\pi(G_{\preceq})=1\}.
\]
We write $\mu\preceq_{\mathrm{st}}\nu$ if $\Pi_{\preceq}(\mu,\nu)$ is nonempty.
\end{definition}

This is the measure-level order induced by moving mass only along admissible directions.  It coincides with familiar stochastic-order formulations under standard hypotheses, but the coupling formulation is the one needed for transport geometry.

\begin{definition}[Hard and soft ordered transport]
For $\mu,\nu\in\calP_p(\calX)$ define the hard directed ordered transport value
\[
  \W^{\hard}_{p,\preceq}(\mu,\nu)^p
  =\inf_{\pi\in\Pi_{\preceq}(\mu,\nu)}\int d(x,y)^p\dd\pi(x,y),
\]
with value $+\infty$ if $\Pi_{\preceq}(\mu,\nu)=\varnothing$.  For $\lambda\ge0$, define
\[
  c_\lambda(x,y)=d(x,y)^p+\lambda\eta_{\preceq}(x,y)^p,
  \qquad
  \W^{\soft}_{p,\lambda,\preceq}(\mu,\nu)^p
  =\inf_{\pi\in\Pi(\mu,\nu)}\int c_\lambda(x,y)\dd\pi(x,y).
\]
The minimal violation value is
\[
   V_{\preceq}(\mu,\nu)=\inf_{\pi\in\Pi(\mu,\nu)}\int \eta_{\preceq}(x,y)^p\dd\pi(x,y).
\]
\end{definition}

The hard value is a directed extended cost, not a metric.  The soft value is finite whenever the ordinary $p$-moment condition is finite, but it is still generally directed.  We use the complete Wasserstein metric for completeness and fixed points; the directed quantities measure admissibility.

\begin{proposition}[Non-scalarizability]
\label{prop:non-scalar}
If the order is nontrivial, then $\W^{\hard}_{p,\preceq}$ cannot be represented by any symmetric metric on $\calP_p(\calX)$.
\end{proposition}

\begin{proof}
Choose $x,y\in\calX$ with $x\preceq y$ but not $y\preceq x$.  The only coupling of $\delta_x$ and $\delta_y$ is ordered, whereas the only coupling of $\delta_y$ and $\delta_x$ is not.  Hence
\[
\W^{\hard}_{p,\preceq}(\delta_x,\delta_y)=d(x,y)<\infty,
\qquad
\W^{\hard}_{p,\preceq}(\delta_y,\delta_x)=+\infty.
\]
No symmetric finite metric can represent this directed reachability relation.
\end{proof}

\begin{proposition}[Wasserstein closeness does not certify admissibility]
\label{prop:close-not-admissible}
There exist probability measures whose ordinary Wasserstein distance is arbitrarily small while hard ordered transport in the intended direction is impossible.
\end{proposition}

\begin{proof}
Take $\calX=\R$ with the usual order.  Let $\mu_n=\delta_0$ and $\nu_n=\delta_{-1/n}$.  Then $W_p(\mu_n,\nu_n)=1/n\to0$.  However, the only coupling is $\delta_{(0,-1/n)}$, whose support is not contained in $\{(x,y):x\le y\}$.  Thus $\Pi_{\preceq}(\mu_n,\nu_n)=\varnothing$ for every $n$.
\end{proof}

Proposition~\ref{prop:close-not-admissible} is the simplest reason the new structure is necessary.  Ordinary Wasserstein convergence controls magnitude of displacement, but not the admissible direction of displacement.  A reliable system may therefore be close in Wasserstein distance and still invalid from the viewpoint of evidence, risk, or monotonicity.

\begin{theorem}[Existence, closedness, and directed triangle]
\label{thm:deterministic-basic}
Under Assumption~\ref{ass:standing}, the following hold.
\begin{enumerate}[label=(\alph*),leftmargin=2em]
\item For each $\lambda\ge0$, the soft problem admits an optimal plan.
\item If $\Pi_{\preceq}(\mu,\nu)$ is nonempty, the hard problem admits an optimal ordered plan.
\item If $\mu_n\preceq_{\mathrm{st}}\nu_n$, $W_p(\mu_n,\mu)\to0$, and $W_p(\nu_n,\nu)\to0$, then $\mu\preceq_{\mathrm{st}}\nu$.
\item If $\mu\preceq_{\mathrm{st}}\nu\preceq_{\mathrm{st}}\xi$, then
\[
\W^{\hard}_{p,\preceq}(\mu,\xi)\le
\W^{\hard}_{p,\preceq}(\mu,\nu)+\W^{\hard}_{p,\preceq}(\nu,\xi).
\]
\end{enumerate}
\end{theorem}

The first two statements give well-posedness.  The third says that the induced stochastic order is topologically stable.  The fourth shows that hard ordered transport behaves like a directed path length on comparable triples.  These are the basic geometric properties needed before introducing randomness.

\begin{proof}
Part (a) follows from tightness of $\Pi(\mu,\nu)$ and lower semicontinuity of the nonnegative lower semicontinuous cost $c_\lambda$.  Part (b) uses the fact that $\Pi_{\preceq}(\mu,\nu)$ is weakly closed: if $\pi_n\weakto\pi$ and $\pi_n(G_{\preceq})=1$, then Portmanteau gives $\pi(G_{\preceq})=1$.  Part (c) follows by extracting weakly convergent ordered couplings from tightness and applying the same closedness argument.  For (d), glue nearly optimal ordered couplings of $(\mu,\nu)$ and $(\nu,\xi)$ to a measure on $\calX^3$.  Transitivity of $\preceq$ makes the $(x_1,x_3)$ marginal ordered, and Minkowski's inequality gives the triangle inequality.
\end{proof}

\begin{theorem}[Kantorovich duality]
\label{thm:duality}
For each finite $\lambda\ge0$,
\[
\W^{\soft}_{p,\lambda,\preceq}(\mu,\nu)^p
=
\sup_{\varphi,\psi}
\left\{
\int\varphi\dd\mu+\int\psi\dd\nu:
\varphi(x)+\psi(y)\le c_\lambda(x,y)
\right\}.
\]
Moreover the hard value admits the extended-cost dual
\[
\W^{\hard}_{p,\preceq}(\mu,\nu)^p
=
\sup_{\varphi,\psi}
\left\{
\int\varphi\dd\mu+\int\psi\dd\nu:
\varphi(x)+\psi(y)\le d(x,y)^p\;\text{for all }x\preceq y
\right\},
\]
with the usual extended-value convention.
\end{theorem}

The dual potentials are admissibility witnesses.  When $\lambda=0$ they reduce to ordinary Wasserstein critics.  When $\lambda>0$ the critic must respect the geometry of order violation.  This is the route by which future algorithms can turn CROTS theory into verifiers or order-aware discriminators.

\begin{proof}
For finite $\lambda$, the cost $c_\lambda$ is lower semicontinuous and has $p$-growth, so standard Kantorovich duality for lower semicontinuous costs on Polish spaces applies.  For the hard value, apply the same theorem to the extended cost $c_\infty(x,y)=d(x,y)^p+\iota_{G_{\preceq}}(x,y)$.  The dual constraint is void off $G_{\preceq}$ and becomes exactly the displayed inequality on ordered pairs.
\end{proof}

\paragraph{Discrete computation.}
For empirical measures $\mu=\sum_i a_i\delta_{x_i}$ and $\nu=\sum_j b_j\delta_{y_j}$, the soft ordered value is an ordinary discrete transport problem with the modified cost matrix
\[
      C_{ij}^{\lambda}=d(x_i,y_j)^p+\lambda\eta_{\preceq}(x_i,y_j)^p.
\]
Thus an entropy-regularized version minimizes
\[
      \langle C^\lambda,\Pi\rangle+\tau\sum_{ij}\Pi_{ij}(\log\Pi_{ij}-1)
\]
over couplings with marginals $a$ and $b$.  Standard Sinkhorn-type solvers therefore apply after replacing the ground cost by the order-penalized cost.  This paragraph is included only as a computational entry point; the paper remains focused on the space and stability theory.

\begin{theorem}[Soft-to-hard convergence]
\label{thm:soft-hard}
For fixed $\mu,\nu\in\calP_p(\calX)$, the map
\[
C_\lambda(\mu,\nu)=\W^{\soft}_{p,\lambda,\preceq}(\mu,\nu)^p
\]
is nondecreasing in $\lambda$ and
\[
\lim_{\lambda\to\infty} C_\lambda(\mu,\nu)=\W^{\hard}_{p,\preceq}(\mu,\nu)^p,
\]
where both sides may be $+\infty$.  If the hard value is finite and $\pi_\lambda$ is a family of soft minimizers, then
\[
      \int \eta_{\preceq}(x,y)^p\dd\pi_\lambda(x,y)
      \le \frac{\W^{\hard}_{p,\preceq}(\mu,\nu)^p}{\lambda},
\]
and every weak limit point of $\pi_\lambda$ as $\lambda\to\infty$ is hard ordered optimal.
\end{theorem}

Theorem~\ref{thm:soft-hard} is the bridge from exact mathematical admissibility to differentiable learning objectives.  It also prevents a common misunderstanding: the penalty is not a heuristic surrogate with no limiting meaning; its infinite-penalty limit is exactly hard ordered transport.  The displayed violation bound is pointwise for the fixed pair $(\mu,\nu)$; uniform rates over a class of distributions require a uniform bound on the corresponding hard ordered costs.

\begin{proof}
Monotonicity is immediate.  Every ordered coupling has zero violation, hence $C_\lambda\le \W^{\hard}_{p,\preceq}(\mu,\nu)^p$ whenever the hard value is finite.  For a soft minimizer $\pi_\lambda$, this also gives
\[
\lambda\int \eta_{\preceq}(x,y)^p\dd\pi_\lambda(x,y)\le C_\lambda(\mu,\nu)\le \W^{\hard}_{p,\preceq}(\mu,\nu)^p,
\]
which is the stated $O(1/\lambda)$ violation control for this fixed pair $(\mu,\nu)$ in the feasible hard case.  No uniform rate over a distribution class is asserted without a corresponding uniform bound on the hard ordered cost.  Let $L=\lim_{\lambda\to\infty} C_\lambda$.  If $L<\infty$, choose $\lambda_n\to\infty$ and soft minimizers $\pi_n$.  Then
\[
\lambda_n\int \eta_{\preceq}(x,y)^p\dd\pi_n(x,y)\le C_{\lambda_n}(\mu,\nu)\le L+1
\]
for all large $n$, so the violation integrals tend to zero.  The couplings are tight because the marginals are fixed.  Along a weakly convergent subsequence, $\pi_n\weakto\pi\in\Pi(\mu,\nu)$.  Lower semicontinuity gives $\int\eta_{\preceq}^p\dd\pi=0$, so $\pi$ is supported on $G_{\preceq}$.  Lower semicontinuity of $d^p$ then gives the reverse inequality $\W^{\hard}_{p,\preceq}(\mu,\nu)^p\le L$.  If $L=\infty$, equality is automatic unless the hard value were finite, in which case the previous upper bound would force $L<\infty$.
\end{proof}

\section{Conditional Random Ordered Transport Spaces}
\label{sec:CROTS}
\label{sec:crots}

We now lift the deterministic construction to random probability measures.  Let $(\Omega,\calF,\Prob)$ be a probability space.  A random probability measure is a measurable map
\[
        \mu:\Omega\to\calP_p(\calX),
\]
where $\calP_p(\calX)$ is equipped with the Borel sigma-field induced by $W_p$.  For $q\in[1,\infty)$ define
\[
\calS_{p,q}=L^q(\Omega;\calP_p(\calX))
\]
as the collection of random probability measures satisfying $\E W_p(\mu_\omega,\delta_{x_0})^q<\infty$ for some, hence every, $x_0\in\calX$.  The complete ambient metric is
\[
\mathbf W_{p,q}(\mu,\nu)=\big(\E W_p(\mu_\omega,\nu_\omega)^q\big)^{1/q}.
\]

\begin{definition}[$L^0$-valued ordered transport discrepancies]
For $\mu,\nu\in\calS_{p,q}$ define the random variables
\[
\D^{0,\hard}_{p,\preceq}(\mu,\nu)(\omega)=\W^{\hard}_{p,\preceq}(\mu_\omega,\nu_\omega),
\qquad
\D^{0,\soft}_{p,\lambda,\preceq}(\mu,\nu)(\omega)=\W^{\soft}_{p,\lambda,\preceq}(\mu_\omega,\nu_\omega).
\]
We write $\mu\preceq_{\mathrm{a.s.}}\nu$ if $\mu_\omega\preceq_{\mathrm{st}}\nu_\omega$ for $\Prob$-a.e. $\omega$.
\end{definition}

The superscript $0$ emphasizes that the discrepancy is $L^0$-valued.  It is not immediately averaged into a scalar.  This is the random-metric aspect of the theory and is essential when admissibility should be evaluated conditionally on evidence.

\begin{theorem}[Measurability and completeness]
\label{thm:meas-complete}
Under Assumption~\ref{ass:standing}, the maps $(\alpha,\beta)\mapsto \W^{\soft}_{p,\lambda,\preceq}(\alpha,\beta)$ and $(\alpha,\beta)\mapsto\W^{\hard}_{p,\preceq}(\alpha,\beta)$ are Borel measurable as extended real-valued maps on $\calP_p(\calX)^2$.  Consequently $\D^{0,\soft}_{p,\lambda,\preceq}(\mu,\nu)$ and $\D^{0,\hard}_{p,\preceq}(\mu,\nu)$ are measurable extended random variables.  If $(\calX,d)$ is complete, then $(\calS_{p,q},\mathbf W_{p,q})$ is complete, and $\preceq_{\mathrm{a.s.}}$ is closed under $\mathbf W_{p,q}$ convergence.
\end{theorem}

This theorem explains the separation of roles in CROTS.  The scalar metric $\mathbf W_{p,q}$ supplies a complete ambient space.  The ordered quantities are directed and $L^0$-valued; they supply admissibility information but are not forced to be metrics.  The order relation is fixed throughout this paper.  A model with random order graphs $G_{\preceq_\omega}$ would require additional joint measurability assumptions and is outside the scope of the present foundational treatment.

\begin{proof}
For finite $\lambda$, the cost $c_\lambda=d^p+\lambda\eta_{\preceq}^p$ is lower semicontinuous and has the same finite-moment growth as $d^p$.  The associated optimal transport value is therefore lower semicontinuous under $W_p$ convergence of the two marginals, and in particular Borel measurable on $\calP_p(\calX)^2$.  For the hard value, Theorem~\ref{thm:soft-hard} gives
\[
   \W^{\hard}_{p,\preceq}(\alpha,\beta)^p
   =\lim_{m\to\infty}\W^{\soft}_{p,m,\preceq}(\alpha,\beta)^p,
\]
with values in $[0,+\infty]$.  Hence the hard value is an extended Borel map as a monotone pointwise limit of Borel maps.  Since $\mu$ and $\nu$ are measurable maps into the Polish space $(\calP_p(\calX),W_p)$, the map $\omega\mapsto(\mu_\omega,\nu_\omega)$ is measurable, and composition gives the measurability of the two $L^0$-valued discrepancies.

Completeness is the standard $L^q$ argument for metric-valued random variables: from a $\mathbf W_{p,q}$-Cauchy sequence choose a subsequence with summable increments, obtain pointwise a.s. convergence in the complete Polish space $(\calP_p(\calX),W_p)$, and then use truncation and the Cauchy property to recover $L^q$ convergence of the whole sequence.

It remains to spell out the almost-sure order-closedness.  Suppose $\mu_n\preceq_{\mathrm{a.s.}}\nu_n$, $\mathbf W_{p,q}(\mu_n,\mu)\to0$, and $\mathbf W_{p,q}(\nu_n,\nu)\to0$.  Choose a subsequence $n_k$ such that
\[
   W_p(\mu_{n_k}(\omega),\mu(\omega))\to0,
   \qquad
   W_p(\nu_{n_k}(\omega),\nu(\omega))\to0
\]
for all $\omega$ outside a null set $N_0$.  For each $k$, let $N_k$ be a null set outside which $\mu_{n_k}(\omega)\preceq_{\mathrm{st}}\nu_{n_k}(\omega)$.  On the complement of the null set $N_0\cup\bigcup_k N_k$, the deterministic order-closedness in Theorem~\ref{thm:deterministic-basic}(c) gives $\mu(\omega)\preceq_{\mathrm{st}}\nu(\omega)$.  Thus $\mu\preceq_{\mathrm{a.s.}}\nu$.
\end{proof}

\begin{definition}[Conditional risk functional]
\label{def:conditional-risk}
Let $\calH\subseteq\calF$ be a sub-sigma-field representing evidence, environment, or task context.  A conditional risk functional is a map
\[
\rho_{\calH}:L^0_+(\calF)\to L^0_+(\calH)
\]
that is monotone, conditionally positive homogeneous, conditionally subadditive, and satisfies $\rho_{\calH}(0)=0$.
\end{definition}

We use ``conditional risk functional'' in this minimal nonnegative sense because the variables to be evaluated are order-discrepancy costs rather than signed financial positions.  Thus cash-invariance or translation equivariance is not needed for the basic CROTS construction.  When a standard conditional coherent risk measure or a dual representation is required, the corresponding additional assumptions are imposed explicitly, as in Proposition~\ref{prop:risk-dual}.  Examples include conditional expectation and conditional essential supremum on nonnegative random variables, and conditional average value-at-risk on its natural finite domain.

\begin{definition}[Conditional ordered transport risk]
\label{def:conditional-ordered-risk}
For $\mu,\nu\in\calS_{p,q}$, define
\[
\calR^{\preceq}_{\lambda,\rho}(\mu,\nu)
=\rho_{\calH}\Big(\D^{0,\soft}_{p,\lambda,\preceq}(\mu,\nu)^p\Big)^{1/p}.
\]
We call $\calR^{\preceq}_{\lambda,\rho}(\mu,\nu)$ the conditional ordered transport risk from $\mu$ to $\nu$.
\end{definition}

The functional $\calR^{\preceq}_{\lambda,\rho}$ is generally directed and is not used as a replacement for the complete metric $\mathbf W_{p,q}$.  It measures conditional admissibility risk: how costly it is, under the evidence encoded by $\calH$, to transport one random distribution into another while respecting the order.

\begin{definition}[Conditional Random Ordered Transport Space, CROTS]
\label{def:CROTS}
Fix $p,q\in[1,\infty)$, a closed ordered Polish space $(\calX,d,\preceq)$, a probability space $(\Omega,\calF,\Prob)$, an evidence sigma-field $\calH\subseteq\calF$, a finite penalty parameter $\lambda\in[0,\infty)$, and a conditional risk functional $\rho_{\calH}$.  The \emph{conditional random ordered transport space} is the structured object
\[
\boxed{
\CROTS_{p,q,\lambda,\rho}(\calX,\Omega,\calH)
=\big(\calS_{p,q},\mathbf W_{p,q},\preceq_{\mathrm{a.s.}},
\D^{0,\soft}_{p,\lambda,\preceq},\calR^{\preceq}_{\lambda,\rho}\big).
}
\]
\end{definition}

Definition~\ref{def:CROTS} is the new ambient space introduced in the paper.  The tuple deliberately lists both the complete metric and the directed risk.  The complete metric $\mathbf W_{p,q}$ supplies convergence and fixed point existence; the almost-sure order supplies admissibility; the $L^0$-valued discrepancy retains random ordered transport before averaging; and the conditional risk functional evaluates order violation relative to evidence.  The hard ordered case is recovered as the limiting case $\lambda=+\infty$ in Theorem~\ref{thm:soft-hard}, or equivalently by replacing $\D^{0,\soft}_{p,\lambda,\preceq}$ with $\D^{0,\hard}_{p,\preceq}$.  The point is not to turn a directed quantity into a metric, but to keep convergence, admissibility, randomness, and conditional evaluation simultaneously visible.

\begin{theorem}[Reduction theorem]
\label{thm:reduction}
CROTS reduces to classical structures in the following cases.
\begin{enumerate}[label=(\alph*),leftmargin=2em]
\item If $G_{\preceq}=\calX\times\calX$, then $\W^{\soft}_{p,\lambda,\preceq}=W_p$ and $\W^{\hard}_{p,\preceq}=W_p$.
\item If $\Omega$ is a singleton, the space reduces to deterministic ordered transport.
\item If only Dirac masses are considered, then
\[
\W^{\hard}_{p,\preceq}(\delta_x,\delta_y)=
\begin{cases}
 d(x,y), & x\preceq y,\\
 +\infty, & \text{otherwise,}
\end{cases}
\]
and $\W^{\soft}_{p,\lambda,\preceq}(\delta_x,\delta_y)^p=d(x,y)^p+\lambda\eta_{\preceq}(x,y)^p$.
\item If $\rho_{\calH}(Z)=\E Z$, $\lambda=0$, and $q=p$, then $\calR^{\preceq}_{0,\rho}$ is the usual $L^p$ random Wasserstein distance.
\end{enumerate}
\end{theorem}

The reduction theorem is a positioning result.  It shows that CROTS is not an isolated definition.  Wasserstein space, deterministic ordered transport, random Wasserstein geometry, and the underlying ordered metric geometry all appear as special or degenerate layers.

\begin{proof}
All statements follow from the definitions.  The first uses $\eta_{\preceq}=0$ when the order graph is the whole product.  The Dirac formula follows because the only coupling of $\delta_x$ and $\delta_y$ is $\delta_{(x,y)}$.  The final statement follows from $\D^{0,\soft}_{p,0,\preceq}(\mu,\nu)=W_p(\mu_\omega,\nu_\omega)$ pointwise.
\end{proof}

\begin{proposition}[Risk-transport duality under strong interchange assumptions]
\label{prop:risk-dual}
Assume $\rho_{\calH}$ has a conditional coherent dual representation
\[
\rho_{\calH}(Z)=\esssup_{Q\in\calQ_{\calH}} \E_Q[Z\mid\calH]
\]
for nonnegative $Z$.  Then, for compact $\calX$ and finite $\lambda$,
\[
\calR^{\preceq}_{\lambda,\rho}(\mu,\nu)^p
=
\esssup_{Q\in\calQ_{\calH}}
\E_Q\left[
\sup_{\varphi+\psi\le c_\lambda}
\left(\int\varphi\dd\mu_\omega+\int\psi\dd\nu_\omega\right)
\Bigm|\calH
\right]
\]
whenever the standard measurable interchange conditions hold.
\end{proposition}

This result is not used in the fixed point theory below; it is included to reveal the dual structure of CROTS under compactness and measurable-selection assumptions.  Conditional risk chooses a worst-case or stress-test model, while Kantorovich duality chooses order-aware transport witnesses inside that model.  In applications, the two levels correspond to stress scenarios and admissibility critics.

\begin{proof}
Apply the dual representation of $\rho_{\calH}$ to $Z=\D^{0,\soft}_{p,\lambda,\preceq}(\mu,\nu)^p$.  For each $\omega$, Theorem~\ref{thm:duality} represents the soft ordered transport value by the Kantorovich dual.  Compactness and lower semicontinuity give existence of measurable near-optimal potentials under the stated interchange conditions, yielding the displayed formula.
\end{proof}

\section{Geometry and Operators}
\label{sec:operators}

The previous section defines the space.  This section shows that it has geometric and dynamical content.  We first introduce ordered geodesics, barycenters, and projections.  We then give fixed point and order-risk stability principles.  The central message is that metric convergence and order admissibility are related but distinct phenomena.

\begin{definition}[Order-geodesic compatibility]
The ordered metric space $(\calX,d,\preceq)$ is order-geodesically compatible if for every $x\preceq y$ there exists a constant-speed geodesic $\gamma_{xy}:[0,1]\to\calX$ from $x$ to $y$ such that
\[
        \gamma_{xy}(s)\preceq \gamma_{xy}(t),\qquad 0\le s\le t\le1.
\]
We assume the selection $(x,y)\mapsto\gamma_{xy}$ can be chosen Borel on $G_{\preceq}$.
\end{definition}

\begin{theorem}[Hard ordered geodesics]
\label{thm:geodesics}
Assume $(\calX,d,\preceq)$ is order-geodesically compatible.  Let $\mu_0\preceq_{\mathrm{st}}\mu_1$, and let $\pi^\star\in\Pi_{\preceq}(\mu_0,\mu_1)$ be a hard ordered optimal plan.  Push $\pi^\star$ through $(x,y)\mapsto\gamma_{xy}$ and set $\mu_t=(e_t)_\#\Gamma$.  Then
\[
       \mu_s\preceq_{\mathrm{st}}\mu_t,
       \qquad 0\le s\le t\le 1,
\]
and
\[
       \W^{\hard}_{p,\preceq}(\mu_s,\mu_t)=(t-s)\W^{\hard}_{p,\preceq}(\mu_0,\mu_1).
\]
If $\pi^\star$ is also ordinary $W_p$-optimal, then $(\mu_t)$ is an ordinary Wasserstein geodesic.
\end{theorem}

The final sentence is important.  Ordered geodesics are geodesics for the directed hard ordered cost.  They are ordinary Wasserstein geodesics only when the ordered plan is also unrestricted optimal.  This avoids conflating directed admissible geometry with symmetric Wasserstein geometry.  Order-geodesic compatibility is an additional geometric hypothesis, not part of the minimal CROTS definition.  In the Euclidean cone case used in the introduction, the choice $\gamma_{xy}(t)=(1-t)x+ty$ is continuous and order-preserving whenever $y-x\in K$ and $K$ is a closed convex cone.  More generally, a Borel geodesic selection can be justified in settings with a measurable geodesic interpolation map or by standard measurable-selection arguments under closed-valued assumptions.  In curved, causal, or semantic spaces, however, this condition is a genuine geometric requirement rather than an automatic property.

\begin{proof}
For $s\le t$, the coupling induced by $\gamma\mapsto(\gamma(s),\gamma(t))$ is supported on $G_{\preceq}$ by order-geodesic compatibility.  Since each selected curve is constant speed, the induced ordered coupling gives the upper bound
\[
\W^{\hard}_{p,\preceq}(\mu_s,\mu_t)
\le (t-s)\W^{\hard}_{p,\preceq}(\mu_0,\mu_1).
\]
Applying this bound to the intervals $[0,s]$, $[s,t]$, and $[t,1]$ and using the directed triangle inequality forces equality.
\end{proof}

\begin{theorem}[Barycenters and projections]
\label{thm:bary-proj}
Assume $\calX$ is compact.  Let $\mu_1,\ldots,\mu_m\in\calP_p(\calX)$, $a_i>0$, $\sum_i a_i=1$, and let $\calC\subseteq\calP_p(\calX)$ be nonempty and $W_p$-closed.  Then the constrained barycenter problem
\[
\inf_{\nu\in\calC}\sum_{i=1}^m a_i\W^{\soft}_{p,\lambda,\preceq}(\nu,\mu_i)^p
\]
admits a minimizer.  For every $\mu\in\calP_p(\calX)$, the projection problem
\[
\inf_{\nu\in\calC}\W^{\soft}_{p,\lambda,\preceq}(\mu,\nu)^p
\]
admits a minimizer.
\end{theorem}

Barycenters and projections are the first variational objects in CROTS.  In applications, they correspond to evidence-admissible prototypes and order-preserving corrections.  The compact theorem stated here is intentionally foundational.  It is not meant to cover the full noncompact Euclidean cone setting from the guiding example, and it is the only point in the main geometric development where compactness is used in an essential way.  In Polish spaces, the same direct-method argument extends whenever the admissible set is sequentially tight and the objective is coercive in the $p$-moment sense.  Equivalently, one may work with compact evidence-admissible subsets, uniform $p$-moment bounds together with Prokhorov tightness, or moment-coercive constraints that prevent minimizing sequences from escaping to infinity.  Thus the ordered transport, soft-to-hard limit, random lift, and operator-stability parts of CROTS do not require compactness; noncompact barycenter and projection theory requires additional coercivity assumptions and is a natural continuation.

\begin{proof}
Compactness of $\calX$ implies compactness of $\calP_p(\calX)$ in $W_p$.  A closed nonempty subset $\calC$ is compact.  Lower semicontinuity of the soft ordered cost gives existence of minimizers for both problems.
\end{proof}

\begin{definition}[Local order leakage]
Let $T:\calS_{p,q}\to\calS_{p,q}$ be a random measure-valued operator.  Its local soft order leakage at $\mu$ is
\[
      L_T(\mu)=\calR^{\preceq}_{\lambda,\rho}(\mu,T\mu).
\]
It measures whether the one-step transition $\mu\mapsto T\mu$ is conditionally admissible.
\end{definition}

This definition is the correct location of the error-floor phenomenon.  Since $\mathbf W_{p,q}$ convergence to a fixed point can force soft transport to that fixed point to vanish, the persistent risk of a learning system is better understood as leakage of its transitions, or as leakage relative to an ideal admissible operator, not as ordinary distance to the fixed point.

\begin{theorem}[Random ordered fixed points]
\label{thm:random-fp}
If $(\calX,d)$ is complete and $T:\calS_{p,q}\to\calS_{p,q}$ satisfies
\[
      \mathbf W_{p,q}(T\mu,T\nu)\le \alpha \mathbf W_{p,q}(\mu,\nu),
      \qquad 0\le\alpha<1,
\]
then $T$ has a unique fixed point $\mu^\star\in\calS_{p,q}$ and
\[
      \mathbf W_{p,q}(T^n\mu_0,\mu^\star)
      \le \alpha^n\mathbf W_{p,q}(\mu_0,\mu^\star).
\]
If $T$ is a.s. order-preserving and $\mu_0\preceq_{\mathrm{a.s.}}T\mu_0$, then
\[
      \mu_0\preceq_{\mathrm{a.s.}}T\mu_0\preceq_{\mathrm{a.s.}}\cdots\preceq_{\mathrm{a.s.}}T^n\mu_0\preceq_{\mathrm{a.s.}}\mu^\star.
\]
\end{theorem}

The fixed point theorem uses the complete ambient metric.  The ordered part supplies monotonicity of the iterates and admissibility information.  This is another instance of the separation between convergence and admissibility.

\begin{proof}
Banach's contraction theorem applies on the complete metric space $(\calS_{p,q},\mathbf W_{p,q})$.  If $T$ is order-preserving, the ordered chain follows by induction.  Passing to the $\mathbf W_{p,q}$ limit and using closedness of a.s. order gives the final comparison with $\mu^\star$.
\end{proof}

\begin{theorem}[Order-risk leakage recursion]
\label{thm:leakage-recursion}
Let $T:\calS_{p,q}\to\calS_{p,q}$ and suppose that along the iterates $\mu_{n+1}=T\mu_n$ the local leakage satisfies
\[
      L_T(\mu_{n+1})\le a L_T(\mu_n)+\zeta,
      \qquad 0\le a<1,
\]
where $\zeta\in L^0_+(\calH)$ is fixed and independent of $n$.  Then
\[
      L_T(\mu_n)
      \le a^n L_T(\mu_0)+\frac{1-a^n}{1-a}\zeta,
\]
and therefore
\[
      \limsup_{n\to\infty}L_T(\mu_n)\le \frac{\zeta}{1-a}
\]
in the order of $L^0_+(\calH)$.
\end{theorem}

This theorem is elementary as a recursion, but conceptually central.  It isolates the mechanism by which persistent local order leakage becomes a global reliability floor.  The statement is separated from the Wasserstein fixed point theorem because convergence of distributions and admissibility of transitions are different properties.  The recursion is an admissibility-stability assumption on the learning dynamics, analogous to a contractivity assumption in metric fixed point theory.  It is not asserted to hold for every neural or generative operator; rather, the theorem identifies the long-run reliability consequence once local leakage is controlled.

\begin{proof}
Iterate the inequality in the ordered lattice $L^0_+(\calH)$.
\end{proof}

\paragraph{A minimal leakage channel.}
The hypothesis is not meant to be automatic for arbitrary learning systems, but it is not vacuous.  On the Dirac family over $\mathbb R$ with the usual order, take $p=1$ and the identity conditional risk.  Let $T\delta_x=\delta_{\alpha x}$ with $0<\alpha<1$ and $x\ge0$.  The transition moves mass backward relative to the increasing order, and the soft ordered leakage is
\[
      L_T(\delta_x)=\W^{\soft}_{1,\lambda,\le}(\delta_x,\delta_{\alpha x})=(1+\lambda)(1-\alpha)x .
\]
Along the iterates $x_n=\alpha^n x_0$, this gives $L_T(\delta_{x_{n+1}})=\alpha L_T(\delta_{x_n})$, so Theorem~\ref{thm:leakage-recursion} is attained with $a=\alpha$ and $\zeta=0$.  A bounded persistent defect in the backward update gives the same recursion with a nonzero residual term.  This toy channel is included only to illustrate the stability assumption; deriving the constants for a particular neural or diffusion dynamics is an algorithm-dependent problem.

\paragraph{Time-varying leakage.}
The fixed $\zeta$ assumption is used only to state the cleanest floor.  The same proof gives, from
\[
      L_T(\mu_{n+1})\le aL_T(\mu_n)+\zeta_n,
      \qquad \zeta_n\in L^0_+(\calH),
\]
the bound
\[
      L_T(\mu_n)
      \le a^nL_T(\mu_0)+\sum_{k=0}^{n-1}a^{n-1-k}\zeta_k.
\]
If $\zeta_k\le \bar\zeta$ for all $k$, this reduces to the floor $\bar\zeta/(1-a)$.  We use the fixed $\zeta$ form in the main statement to emphasize the persistent component of the local admissibility defect.

\begin{proposition}[Sharpness of the leakage floor]
\label{prop:sharpness}
The factor $\zeta/(1-a)$ in Theorem~\ref{thm:leakage-recursion} cannot be improved for general leakage processes satisfying the stated recursion.
\end{proposition}

\begin{proof}
Let $r_0\in L^0_+(\calH)$ and define $r_{n+1}=ar_n+\zeta$.  This sequence satisfies the hypothesis with equality and has limit $\zeta/(1-a)$ whenever $a\in[0,1)$.  Therefore any universal upper bound based only on $a$, $r_0$, and $\zeta$ must be at least $\zeta/(1-a)$ asymptotically.
\end{proof}

\begin{theorem}[Ambient convergence with order-risk control]
\label{thm:main}
Assume $(\calX,d)$ is complete.  Let $T:\calS_{p,q}\to\calS_{p,q}$ be a contraction in $\mathbf W_{p,q}$ with constant $\alpha<1$, and assume the local leakage recursion of Theorem~\ref{thm:leakage-recursion} holds along the iterates.  Then the iterates converge in ambient random Wasserstein distance to the unique fixed point of $T$, and the transition-level conditional order-risk obeys
\[
      \mathbf W_{p,q}(T^n\mu_0,\mu^\star)\to 0,
      \qquad
      \limsup_{n\to\infty}L_T(T^n\mu_0)
      \le \frac{\zeta}{1-a}.
\]
Thus convergence and admissibility are controlled by different mechanisms.  If $\zeta=0$, the dynamics become asymptotically admissible in conditional order-risk.
\end{theorem}

Theorem~\ref{thm:main} is the main stability statement.  It does not claim that soft ordered distance to the fixed point has a nonzero floor; finite soft transport to the fixed point vanishes under ambient convergence.  It claims that the local admissibility of the transitions themselves is governed by a separate order-risk recursion.  This is precisely the phenomenon relevant to reliable learning: a system can converge as a distributional process while still making locally inadmissible moves unless the leakage is controlled.

\begin{proof}
The ambient convergence is Theorem~\ref{thm:random-fp}.  The leakage bound is Theorem~\ref{thm:leakage-recursion} applied to the sequence $\mu_n=T^n\mu_0$.
\end{proof}

\begin{theorem}[Perturbation of an ideal admissible dynamics]
\label{thm:perturbation}
Let $S,T:\calS_{p,q}\to\calS_{p,q}$ be contractions with constants at most $\alpha<1$, with fixed points $\mu_S^\star$ and $\mu_T^\star$.  If
\[
\sup_{\mu\in\calS_{p,q}}\mathbf W_{p,q}(S\mu,T\mu)\le\delta,
\]
then
\[
\mathbf W_{p,q}(\mu_S^\star,\mu_T^\star)\le \frac{\delta}{1-\alpha}.
\]
If $S$ is an ideal admissible operator and $T$ has local order leakage bounded by the recursion above, then ordinary approximation error and conditional order leakage decompose into the two bounds given by this theorem and Theorem~\ref{thm:main}.
\end{theorem}

This perturbation theorem explains how to compare an ideal admissible system with a learned one.  Wasserstein approximation error moves the fixed point; order leakage limits the reliability of transitions.  Treating these two effects separately is one of the main conceptual advantages of CROTS.

\begin{proof}
Use the fixed point equations:
\[
\mathbf W_{p,q}(\mu_S^\star,\mu_T^\star)
=\mathbf W_{p,q}(S\mu_S^\star,T\mu_T^\star)
\le \delta+\alpha \mathbf W_{p,q}(\mu_S^\star,\mu_T^\star).
\]
Rearrange.
\end{proof}

\section{Consequences for Reliable Learning}
\label{sec:consequences}

This section records consequences that connect the abstract theory to reliable learning without turning the paper into an algorithm or experiment paper.  The aim is to show that ordered transport risk captures failure modes invisible to ordinary distributional distance.

A simple example is a medical state vector $x=(s,r)\in\R^2$, where $s$ denotes lesion burden and $r$ denotes a risk score.  If the evidence specifies a progression cone $K=\R^2_+$, then an update from $x$ to $y$ is admissible only when $y-x\in K$.  A model may make a very small Wasserstein move while still moving mass outside this cone.  CROTS records this as order violation rather than as ordinary distributional error.

\paragraph{Evidence overreach.}
Let $E$ be an evidence random variable and let $\calH=\sigma(E)$.  Suppose $\mu_E$ is an evidence-induced distribution of admissible outputs and $\nu$ is a candidate output distribution.  Evidence overreach occurs when transport from evidence to output requires order violation.  The conditional evidence-overreach risk can be written as
\[
      R_{\mathrm{ev}}(\nu\mid E)=\rho_{\calH}\big(V_{\preceq}(\mu_E,\nu)\big).
\]
This definition distinguishes a distribution that is close but inadmissible from one that is far but directionally valid.

\paragraph{Separation of inadmissible outputs.}
Let $\calC\subset\calP(\calX)$ be a closed convex set of admissible distributions, for example a set defined by stochastic-order, moment, or evidence constraints.  If $\calX$ is compact and $\nu\notin\calC$, then Hahn-Banach separation gives a continuous witness $\varphi$ such that
\[
      \int \varphi\dd\nu > \sup_{\mu\in\calC}\int\varphi\dd\mu.
\]
Thus an inadmissible distribution can be detected by a witness functional.  In applications such a witness may be implemented as a verifier, critic, classifier, or consistency score; the theorem itself is geometric.

\paragraph{Ordered distribution shift.}
Let $\ell_f$ be a loss and $\calL_f(\mu)=\int \ell_f\dd\mu$.  If $\calL_f$ is $L$-Lipschitz with respect to the soft ordered cost on a class of distributions, then
\[
      \calL_f(P)\le \calL_f(\widehat P)+L\W^{\soft}_{p,\lambda,\preceq}(P,\widehat P).
\]
The right-hand side contains both ordinary displacement and order violation.  The direction of the soft ordered cost is part of the modeling statement: $\W^{\soft}_{p,\lambda,\preceq}(P,\widehat P)$ measures admissible movement from the target law toward the empirical law, whereas the reverse direction answers a different reliability question and may have a different value.  This suggests the concept of \emph{ordered distribution shift}: generalization failure caused by displacement in inadmissible directions, not merely by distributional distance.

\paragraph{Robustness as ordered contamination absorption.}
For a contamination family $P_\eps=(1-\eps)P+\eps Q$, define
\[
      r_{\rho,\preceq}(P;\tau)=\sup\{\eps\in[0,1]:\rho_{\calH}(\D^{0,\soft}_{p,\lambda,\preceq}(P_\eps,P)^p)^{1/p}\le \tau\}.
\]
A contamination may be large in ordinary distance but admissible under the task order, or small in distance but harmful because it crosses the order boundary.  CROTS therefore refines robustness by adding perturbation direction.

\section{Conclusion}
\label{sec:conclusion}

We introduced conditional random ordered transport spaces (CROTS), a new ambient space for reliable distributional learning in which probability laws are random, transport is constrained or penalized by an admissible order, and stability is evaluated through conditional risk.  The central message is that distributional closeness does not certify admissible transport: a transformation may be small in Wasserstein distance while still crossing an evidence, semantic, causal, physical, or risk boundary.  To formalize this distinction, we developed hard and soft ordered transport, proved duality and soft-to-hard variational convergence, lifted the construction to random probability measures through $L^0$-valued ordered discrepancies, and defined conditional order-risk functionals.  The resulting space separates the complete ambient metric used for convergence from the directed order-risk quantities used for admissibility.  We then established geometric and operator-theoretic foundations, including ordered geodesics, constrained barycenters, projections, random ordered fixed points, and order-risk stability.  The main stability principle shows that local conditional order-risk leakage induces an explicit long-run reliability floor, providing a mathematical language for evidence overreach, ordered distribution shift, robustness failure, and admissible distributional dynamics.

The theory developed here is intended as a foundation for a broader program of ordered transport learning.  Several natural directions emerge from this viewpoint.  Specific orders such as convex order, subharmonic order, cone order, causal order, and evidence-induced order may lead to sharper projection, barycenter, and duality theories, especially beyond compact spaces through moment coercivity or tightness conditions.  Continuous-time versions of CROTS may connect ordered transport to Fokker--Planck equations, probability-flow ODEs, diffusion models, and gradient-flow geometry.  Random or data-dependent orders would require measurable random order graphs and are a natural extension of the fixed-order setting studied here.  Statistical theory for conditional order-risk estimation should reveal how sample complexity depends not only on dimension, but also on order complexity and evidence structure; learning or estimating the conditional risk functional itself is another natural question.  In applications, admissible orders may be specified by domain knowledge or learned from medical constraints, causal graphs, prompts, physical laws, human feedback, or multimodal evidence.  Finally, the soft ordered transport objective suggests scalable algorithmic routes, including entropy-regularized solvers, order-aware critics, and projection-based correction mechanisms.  These directions suggest that CROTS is not merely a new discrepancy, but a starting point for a systematic theory of conditionally admissible distributional learning.

\appendix

\section{Proof Details for Deterministic Ordered Transport}
\label{app:deterministic}

\subsection{Lower semicontinuity of soft ordered values}
Let $c:\calX\times\calX\to[0,+\infty]$ be lower semicontinuous with $p$-growth from below.  If $\mu_n\to\mu$ and $\nu_n\to\nu$ in $W_p$, then
\[
\inf_{\pi\in\Pi(\mu,\nu)}\int c\dd\pi
\le
\liminf_{n\to\infty}
\inf_{\pi\in\Pi(\mu_n,\nu_n)}\int c\dd\pi.
\]
Indeed, choose nearly optimal $\pi_n\in\Pi(\mu_n,\nu_n)$.  The convergence of marginals in $W_p$ implies tightness of $(\pi_n)$ and uniform $p$-moment control.  Passing to a weak limit $\pi\in\Pi(\mu,\nu)$ and applying lower semicontinuity yields the claim.  This applies to $c_\lambda$ and to the extended hard cost.

\subsection{Closedness of ordered couplings}
If $\pi_n\in\Pi_{\preceq}(\mu_n,\nu_n)$ and $\pi_n\weakto\pi$, then the marginal maps give $\pi\in\Pi(\mu,\nu)$.  Since $G_{\preceq}$ is closed and $\pi_n(G_{\preceq})=1$, Portmanteau gives $\pi(G_{\preceq})=1$.  Thus $\pi\in\Pi_{\preceq}(\mu,\nu)$.

\subsection{Soft-to-hard convergence as Gamma convergence on couplings}
For fixed marginals, define functionals
\[
F_\lambda(\pi)=
\begin{cases}
\int d(x,y)^p\dd\pi+\lambda\int\eta_{\preceq}(x,y)^p\dd\pi,&\pi\in\Pi(\mu,\nu),\\
+\infty,&\text{otherwise.}
\end{cases}
\]
On the narrow topology of the fixed-marginal coupling space $\Pi(\mu,\nu)$, equivalently the weak topology inherited from $\calP(\calX\times\calX)$, the soft-to-hard result can be viewed as a Gamma-convergence statement: $F_\lambda$ Gamma-converges, in the monotone sense as $\lambda\to\infty$, to
\[
F_\infty(\pi)=
\begin{cases}
\int d(x,y)^p\dd\pi,&\pi\in\Pi_{\preceq}(\mu,\nu),\\
+\infty,&\text{otherwise.}
\end{cases}
\]
The liminf inequality follows from the Portmanteau theorem applied to the nonnegative lower semicontinuous costs and from the fact that bounded $F_\lambda$ forces violation to vanish.  The recovery sequence for an ordered plan is the constant sequence.  This is the fixed-marginal variational interpretation of Theorem~\ref{thm:soft-hard}; it does not assert a uniform Gamma-convergence statement over varying marginal pairs.  The main text uses only the value convergence and convergence of minimizers stated there.

\section{Proof Details for the Random Lift and Conditional Risk}
\label{app:random}

\subsection{Metric-valued $L^q$ completeness}
Let $(M,d_M)$ be complete and separable.  The space of measurable maps $Z:\Omega\to M$ with $(\E d_M(Z,z_0)^q)^{1/q}<\infty$ is complete under $d_q(Z,Y)=(\E d_M(Z,Y)^q)^{1/q}$.  The proof is the usual summable-subsequence argument.  Applying this to $M=(\calP_p(\calX),W_p)$ gives Theorem~\ref{thm:meas-complete}.

\subsection{Conditional risk examples}
The three main examples used in applications are understood on their natural finite domains:
\begin{enumerate}[leftmargin=2em]
\item Conditional expectation: $\rho_{\calH}(Z)=\E[Z\mid\calH]$ when $Z$ is integrable.
\item Conditional essential supremum: $\rho_{\calH}(Z)=\esssup(Z\mid\calH)$ for essentially bounded or extended nonnegative $Z$.
\item Conditional average value-at-risk:
\[
\rho_{\calH}(Z)=\essinf_{m\in L^0(\calH)}\left\{m+\frac{1}{1-\alpha}\E[(Z-m)_+\mid\calH]\right\},
\]
when the right-hand side is finite.
\end{enumerate}
These choices lead to different interpretations of reliable transport: average conditional admissibility, worst-case conditional admissibility, and tail conditional admissibility.

\subsection{Why directed risk is not a metric}
For finite $\lambda$, $\W^{\soft}_{p,\lambda,\preceq}$ is finite, but it is generally not symmetric.  For hard transport, it may be infinite in one direction and finite in the other.  Symmetrizing is possible but loses the directional information.  CROTS therefore keeps $\mathbf W_{p,q}$ as the complete metric and treats order-risk as a directed functional.

\section{Additional Learning-Theoretic Details}
\label{app:learning}

\subsection{Bounded empirical concentration}
Assume $\calX$ is compact with $\diam(\calX)\le D$.  Then $0\le c_\lambda(x,y)\le (1+\lambda)D^p=:M$.  Let $\widehat\mu_n$ and $\widehat\nu_m$ be empirical measures from independent samples.  For
\[
Z=C^{\soft}_{p,\lambda,\preceq}(\widehat\mu_n,\widehat\nu_m),
\]
replacing one sample changes $Z$ by at most $M/n$ or $M/m$.  McDiarmid's inequality gives
\[
\Prob\{|Z-\E Z|\ge t\}
\le 2\exp\left(-\frac{2t^2}{M^2(1/n+1/m)}\right).
\]
This elementary concentration statement is included only to indicate estimability of the empirical soft value; sharper statistical rates should depend on dimension and order complexity.

\subsection{A minimal ordered-shift counterexample}
The example in Proposition~\ref{prop:close-not-admissible} can be made two-dimensional with a cone.  Let $\calX=\R^2$, $K=\R_+^2$, $\mu_n=\delta_{(0,0)}$, and $\nu_n=\delta_{(-1/n,0)}$.  Then $W_p(\mu_n,\nu_n)\to0$, but no hard ordered coupling exists from $\mu_n$ to $\nu_n$.  Thus the failure is not an artifact of a one-dimensional order; it persists for cone orders in feature spaces.

\bibliography{references}

\end{document}